\RequirePackage[OT1]{fontenc}
\documentclass[letterpaper, 10 pt, conference]{ieeeconf}
\IEEEoverridecommandlockouts
\usepackage{amssymb, mathtools, xargs, tensor, units, cite, mathrsfs, algpseudocode, graphicx, xcolor, amsmath, listings, algorithm, tikz, pgfplots, cancel, multicol, geometry, multirow, booktabs}

\usepgfplotslibrary{fillbetween} 
\algnewcommand\algorithmicforeach{\textbf{for each}}
\algdef{S}[FOR]{ForEach}[1]{\algorithmicforeach\ #1\ \algorithmicdo}
\usepackage{pgfplots}
\usepackage{pgfplotstable}
\usepackage{amsmath}
\pgfplotsset{compat=newest}

\pgfplotsset{compat=1.17}
\usepgfplotslibrary{colormaps}
\definecolor{customgreen}{rgb}{0.0, 0.6, 0.0}

\pgfplotsset{compat=1.18}
\usetikzlibrary{3d, calc}
\pgfplotsset{compat=newest}
\usetikzlibrary{arrows.meta}

\newtheorem{assumption}{Assumption}

\newtheorem{remark}{Remark}

\newtheorem{theorem}{Theorem}

\allowdisplaybreaks
\title{\LARGE A Taylor Series Approach to Correct Localization Errors in Robotic Field Mapping using Gaussian Processes}
\author{Muzaffar Qureshi$^{1}$ \and Tochukwu Elijah Ogri$^{1}$  \and Kyle Volle$^{2}$ \and Rushikesh Kamalapurkar$^{1}$
\thanks{This research was supported in part by the Air Force Research Laboratories under contract numbers FA8651-24-1-0019 and FA8651-23-1-0006 and the Office of Naval Research under contract number N00014-21-1-2481. Any opinions, findings, or recommendations in this article are those of the author(s), and do not necessarily reflect the views of the sponsoring agencies.}%
\thanks{$^{1}$ Department of Mechanical and Aerospace Engineering, University of Florida, Gainesville, Florida, USA, email: {\tt\footnotesize \{muzaffar.qureshi, tochukwu.ogri, rkamalapurkar\} @ufl.edu}.}%
\thanks{$^{2}$ Torch Technologies, Shalimar, Florida, USA, email: {
\tt \footnotesize Kyle.Volle@torchtechnologies.com}}}

\begin{document}
\maketitle
\pagestyle{empty}

\begin{abstract}
Gaussian Processes (GPs) are powerful non-parametric Bayesian models for regression of scalar fields, formulated under the assumption that measurement locations are perfectly known and the corresponding field measurements have Gaussian noise. However, many real-world scalar field mapping applications rely on sensor-equipped mobile robots to collect field measurements, where imperfect localization introduces state uncertainty. Such discrepancies between the estimated and true measurement locations degrade GP mean and covariance estimates. To address this challenge, we propose a method for updating the GP models when improved estimates become available. Leveraging the differentiability of the kernel function, a second-order correction algorithm is developed using the precomputed Jacobians and Hessians of the GP mean and covariance functions for real-time refinement based on measurement location discrepancy data. Simulation results demonstrate improved prediction accuracy and computational efficiency compared to full model retraining.
\end{abstract}

\section{Introduction}
Mapping of unknown fields is critical in various scientific and engineering domains, including environmental monitoring, search and rescue missions, and autonomous underwater exploration \cite{muzaffar.Sabella1998,muzaffar.Lin.Liu.ea2019}. Among the various available mapping techniques, Gaussian Process (GP) have emerged as a powerful approach due to their ability to capture complex spatial dependencies and to provide uncertainty quantification \cite{muzaffar.Schulz.Speekenbrink.ea2018,SCC.Rasmussen.Williams2006}. Standard GP regression operates under the assumptions that measurement locations are known and the corresponding measurements have Gaussian noise. 

In many GP-based mapping applications, a sensor-equipped robot navigates through the environment and gathers field measurements \cite{muzaffar.Elhashash.Albanwan.ea2022}. In practice, the true pose of the robot may need to be estimated \cite{muzaffar.Bai.Taylor2020}. When estimating the robot pose, error sources such as unmodeled robot dynamics, environmental disturbances, or miscalibration can cause errors in the estimate. This can result in improper inputs into the GP and corrupt the data if not accounted for \cite{SCC.Rasmussen.Williams2006,muzaffar.Uzzaman.Bai2024}.

In most field mapping problems, global positioning system (GPS) sensors are used to provide localization for the robot during the measurement process. However, GPS signals may be unavailable or unreliable in environments with occlusions, signal interference, or other challenging conditions \cite{muzaffar.Bachrach.Prentice.ea2011, muzaffar.Weiss.Scaramuzza.ea2011, muzaffar.Qureshi.Ogri.ea2025a, muzaffar.Qureshi.Ogri.ea2024}. In the absence of GPS corrections, uncertainties in input locations propagate through the GP model, leading to degraded performance. 

Girard et al. \cite{muzaffar.Girard2004} developed a framework to propagate the uncertainty in the GP model. By modeling the input to the GP as a Gaussian random variable and leveraging the analytical properties of the squared exponential (SE) kernel, their approach yields closed-form expressions for the mean and variance. This method is effective for dynamic systems, finding application in discrete-time nonlinear models \cite{muzaffar.Candela.Girard.ea2003} and iterative $k$-step-ahead time-series forecasting \cite{muzaffar.Girard2004}, as it enables the principled propagation of accumulating uncertainty. The tractability of this framework imposes notable limitations. The reliance on specific, analytically tractable kernels restricts generality, and it becomes intractable if the input uncertainty deviates from the Gaussian assumption. When these assumptions are violated, computationally expensive approximations are required, which in turn complicates the task of hyperparameter estimation in complex or high-dimensional datasets \cite{muzaffar.Tripathy.Bilionis.ea2016}.

As an alternative, Mchutchon and Rasmussen presented the noisy input GP (NIGP) model, which provides an approximate analytical solution for handling input uncertainty \cite{muzaffar.Mchutchon.Rasmussen2011}. Rather than requiring a specific kernel, their method employs a first-order Taylor expansion to linearize the function at each noisy measurement location. This approximation enables the input noise to be recast as a corrective term to the output variance, which is proportional to the squared gradient of the posterior mean. All model hyperparameters, including the input noise variances, are then jointly learned by maximizing the marginal likelihood. The primary limitation of NIGP stems from its first-order approximation, which can become inaccurate for functions exhibiting non-linearities \cite{muzaffar.Mchutchon.Rasmussen2011}. Furthermore, the need to optimize additional hyperparameters for the input noise variances increases the computational complexity of the training process. The model also assumes a constant input noise variance, limiting its efficacy in more general scenarios where measurement location uncertainty varies across the dataset.

While the aforementioned methods to address measurement location uncertainty assume that errors are stochastic, zero-mean, and Gaussian. The objective of these methods is to compensate for this uncertainty by adjusting the covariance function or incorporating a corrective variance term. However, many real-world mapping applications are dominated by systematic and deterministic errors in measurement locations rather than random noise. An example is the accumulating dead-reckoning drift that occurs during navigation in GPS-denied environments, which introduces structured, non-stationary errors into the position estimates \cite{muzaffar.Qureshi.Ogri.ea2025, muzaffar.Bachrach.Prentice.ea2012}. 

In graph-based SLAM systems, as well as Kalman filter and observer-based designs, the arrival of new state information, such as through loop closure or GPS reacquisition, triggers a global optimization that refines the entire trajectory history \cite{muzaffar.Hess.Kohler.ea2016}. This process deterministically replaces prior state estimates with improved ones, without modeling measurement locations as random variables \cite{muzaffar.Placed.Strader.ea2023}. The GP model must then be updated to reflect these revised inputs, but retraining the complete model on the corrected dataset is computationally prohibitive. 

Addressing the challenge of deterministic input errors requires a method to retroactively correct GP-based maps when improved state estimates become available, thereby avoiding computationally prohibitive retraining. The difficulty in designing a GP corrective framework stems from the dependence of the mean and covariance functions on the inverse of the kernel matrix, where a change to a measurement location requires a computationally prohibitive $O(T^3)$ retraining, where $T$ denotes the number of measurement locations. This paper develops a two-stage correction framework: an offline phase that precomputes and stores the Jacobians and Hessians of the GP mean and covariance, and an online phase that applies real-time corrections using this precomputed data as improved location estimates become available. The key contribution of this work is the development of a gradient-based correction method for GP models that corrects deterministic measurement location errors.

\section{Problem Formulation}
Consider a sensor-equipped robot operating in a domain $\mathcal{X} \subset \mathbb{R}^n$ that is tasked with visiting a set of measurement locations to create a map of a scalar field \( f: \mathcal{X} \rightarrow \mathbb{R} \). The field can be modeled by GP as 
\begin{equation}
\hat{f} \sim \text{GP}(\mu, \Sigma), 
\label{eq:GP_model}
\end{equation}  
where $\mu:\mathbb{R}^n \rightarrow \mathbb{R}$ is the mean function and $\Sigma: \mathbb{R}^n \times \mathbb{R}^n \rightarrow \mathbb{R}$ is the covariance function.

To develop the GP model, the robot is programmed to follow a planned trajectory, visiting a set of planned measurement locations $\hat{\mathbf{X}} = \{ \hat{\mathbf{x}}_i \}_{i=1}^{T}$, where $\hat{\mathbf{x}}_i \in \mathcal{X}$ and $T \in \mathbb{N}$ is the total number of measurements. However, due to unmodeled dynamics or localization errors, the trajectory of the robot deviates from the planned trajectory. This results in the robot taking measurements at a set of actual locations denoted as $\mathbf{X} = \{ \mathbf{x}_i\}_{i=1}^{T}$. The robot collects the sensor measurements $\mathbf{Y} =[y_1,\dots, y_T ]^\top \in \mathbb{R}^{T \times 1}$ at these unknown locations, with
\begin{equation}
    y_i = f(\mathbf{x}_i) + \epsilon_{i},
\end{equation}
where $\epsilon_{i} \sim \mathcal{N}(0, \sigma_y^2)$ is a zero-mean Gaussian noise with variance $\sigma_y^2$.

The robot does not account for localization error during training and incorrectly associates the actual measurements $\mathbf{Y}$ with its planned locations $\hat{\mathbf{X}}$. The resulting corrupted GP model is thus trained on the mismatched dataset $\{\hat{\mathbf{X}}, \mathbf{Y}\}$. 

Using the state estimation techniques, the error between the true and planned measurement locations can be computed as
\begin{equation}
\mathbf{x}_i = \hat{\mathbf{x}}_i + \delta_i, \quad \forall i=1,2,\dots,T,
\end{equation}
where $\delta_i \in \mathbb{R}^n$ is the localization error for the $i$-th measurement location. The objective is to develop a correction method for the corrupted GP model using the known localization errors $\delta_i$. The aim is to approximate the predictive performance of an ideal GP model trained on the true dataset ${ \mathbf{X}, \mathbf{Y} }$, assuming perfect localization. Importantly, the correction is performed without retraining the complete GP model or any new measurements. The following assumptions are made to keep the problem tractable.

\begin{assumption}\label{assump:1}
The measurement noise $\epsilon_i$ is modeled as independent and identically distributed (i.i.d.) Gaussian noise with zero mean and variance $\sigma_y^2$, i.e., $\epsilon_i \sim \mathcal{N}(0, \sigma_y^2)$ and $\mathbb{E}[\epsilon_i \epsilon_j] = 0$ for all $i \neq j$.
\end{assumption}

\begin{assumption}\label{assump:2}
The error between the planned and the actual measurement locations is bounded, i.e., $\|\delta_i\| \leq \delta_{\max}, \forall i =1,2,\dots,T$.
\end{assumption}

\section{Gaussian process Regression}
In this paper, the unknown function \( f \) is modeled as a sample of a GP,
with the kernel $k$ selected to be SE kernel \cite{SCC.Rasmussen.Williams2006}, defined as
\begin{equation}
k (\hat{\mathbf{x}}, \hat{\mathbf{x}}') \coloneqq \alpha^2 \exp \left( -\frac{|| \hat{\mathbf{x}} - \hat{\mathbf{x}}'||^2}{2 \beta^2} \right),
\label{eq:kernel}
\end{equation}
where $\alpha^2$ denotes the signal variance and $\beta$ represents the characteristic lengthscale, and $\hat{\mathbf{x}}$,  $\hat{\mathbf{x}}'$ represent any two locations in $
\mathcal{X}$. The SE kernel is selected due to its smoothness properties \cite{SCC.Rasmussen.Williams2006} and its universal approximation property \cite{muzaffar.Park.Sandberg1991}.

Let \( \mathbf{X}_e = \{ \mathbf{x}_{e,1}, \mathbf{x}_{e,2}, \dots, \mathbf{x}_{e,M} \}  \ \subset \mathcal{X} \) be the set of test locations for GP posterior evaluation. The constant test-test kernel matrix ${K}_{e-e} \in \mathbb{R}^{M \times M}$, which depends only on the fixed test locations, can be computed as
\begin{equation}
{K}_{e-e}[i,j] = k(\mathbf{x}_{e,i}, \mathbf{x}_{e,j}).
\label{eq:test_test_kernel}
\end{equation}
The kernel functions, which depend on any general set of $T$ training locations 
$\mathbf{Z} = \{ \mathbf{z}_t \}_{t=1}^T \in \mathcal{X}^T$, where $\mathcal{X}^T \coloneqq \mathcal{X} \times \dots \times \mathcal{X}$ is the $T$-fold Cartesian product of $\mathcal{X}$, can be defined as
\begin{gather}
{K}_{e-\mathrm{T}}(\mathbf{Z})[i,k] := k(\mathbf{x}_{e,i}, \mathbf{z}_k), \label{eq:test_trn_kernel} \\
{K}_{\mathrm{T}-\mathrm{T}}(\mathbf{Z})[j,k] := k\left(\mathbf{z}_j, \mathbf{z}_k \right), \label{eq:trn_trn_kernel}
\end{gather}
where ${K}_{e-\mathrm{T}}: \mathcal{X}^T \to \mathbb{R}^{M \times T}$ and ${K}_{\mathrm{T}-\mathrm{T}}: \mathcal{X}^T \to \mathbb{R}^{T \times T}$ are the test-train and train-train kernel functions using fixed set $\mathbf{X}_e$, respectively. For a given input $\mathbf{Z}$, they evaluate to the corresponding matrices.

Using these definitions, the GP mean function 
\( \mathbf{m}: \mathcal{X}^T \rightarrow \mathbb{R}^M \)
and GP covariance function 
\( \mathbf{S}: \mathcal{X}^T \rightarrow \mathbb{R}^{M \times M} \)
are defined as
\begin{equation}
\mathbf{m}(\mathbf{Z}) \coloneqq 
{K}_{e-\mathrm{T}}(\mathbf{Z}) 
\left( {K}_{\mathrm{T}-\mathrm{T}}(\mathbf{Z}) + \sigma_y^2 \mathbf{I}_T \right)^{-1} 
\mathbf{Y},
\label{eq:mean_function_def}
\end{equation}
\begingroup\medmuskip=0mu\thinmuskip=0mu\thickmuskip=0mu\begin{equation}
\mathbf{S}(\mathbf{Z}) \coloneqq 
\mathbf{K}_{e-e} -
{K}_{e-\mathrm{T}}(\mathbf{Z})
\left( {K}_{\mathrm{T}-\mathrm{T}}(\mathbf{Z}) + \sigma_y^2 \mathbf{I}_T \right)^{-1}
{K}_{e-\mathrm{T}}(\mathbf{Z})^\top,
\label{eq:cov_function_def}
\end{equation}
\endgroup
respectively, where
\( \mathbf{m} \in \mathbb{R}^{M} \) 
is the mean vector,
\( \mathbf{S} \in \mathbb{R}^{M \times M} \) 
is the covariance matrix, and 
\( \mathbf{I}_T \) denotes the identity matrix of size \( T \). The predicted mean \( \hat{\mathcal{M}} \) and covariance \( \hat{\mathcal{S}} \) of the corrupted GP model 
are the evaluations of these functions at the planned measurement locations \( \hat{\mathbf{X}} \) and true measurements  \( \hat{\mathbf{Y}} \)  computed as
\begin{equation}
\hat{\mathcal{M}} \coloneqq 
{K}_{e-\mathrm{T}}(\mathbf{\hat{X}})
\left( {K}_{\mathrm{T}-\mathrm{T}}(\mathbf{\hat{X}}) + \sigma_y^2 \mathbf{I}_T \right)^{-1}
\mathbf{Y},
\label{eq:meanmatrix_vec}
\end{equation}
\begingroup\medmuskip=0mu\thinmuskip=0mu\thickmuskip=0mu\begin{equation}
\hat{\mathcal{S}} \coloneqq 
\mathbf{K}_{e-e} -
{K}_{e-\mathrm{T}}(\mathbf{\hat{X}})
\left( {K}_{\mathrm{T}-\mathrm{T}}(\mathbf{\hat{X}}) + \sigma_y^2 \mathbf{I}_T \right)^{-1}
{K}_{e-\mathrm{T}}(\mathbf{\hat{X}})^\top.
\label{eq:comatrix_mat}
\end{equation}\endgroup
A sensitivity analysis of the Gaussian Process mean and covariance functions is conducted with respect to perturbations in the training input locations. A second-order Taylor series expansion is derived around the planned measurement locations \( \hat{\mathbf{X}} \), enabling efficient correction of the GP posterior under localization error.

\section{Taylor Series Approximation of Mean and Covariance functions}
In GP regression, the mean function $\mathbf{m}$ is computed using a linear combination of the observed measurements, \({\mathbf{Y}}\), weighted by the correlation between the test and the train locations. These weights are derived from the kernel matrix \(K_{e-\mathrm{T}}\) in \eqref{eq:test_trn_kernel} and the inverse of the kernel matrix \(K_{\mathrm{T}-\mathrm{T}}\) in \eqref{eq:trn_trn_kernel}, which both contribute to the sensitivity of the GP mean to changes in the measurement locations. Similarly, the covariance  \({\mathbf{S}}\), which quantifies the uncertainty in the predictions, is governed by the matrices \(K_{e-e}\), \(K_{e-\mathrm{T}}\), and \(K_{\mathrm{T}-\mathrm{T}}\) and quantify correlation between the test and measurement locations. 

To quantify the effect of the error between the true measurement locations $\mathbf{X}$ and the planned locations $\hat{\mathbf{X}}$ on the GP mean and covariance functions, a second-order Taylor series expansion is performed at each of the planned measurement locations $\hat{\mathbf{x}}_i \in \mathbb{R}^n$ for all $i = 1, \dots, T$. The corrected mean vector $\mathcal{M} \coloneqq \mathbf{m}(\mathbf{X}) \in \mathbb{R}^{M \times 1}$ is approximated using a second-order Taylor expansion around $\hat{\mathbf{X}}$ as 
\begin{equation}\label{eq:mean_update}
\mathcal{M} \coloneqq \hat{\mathcal{M}}
+ \sum_{i=1}^{T} \mathbf{J}_M^i \delta_i
+ \frac{1}{2} \sum_{i=1}^{T} \sum_{j=1}^{T} \left[
\begin{array}{c}
\delta_i^\top \mathbf{H}_{M}^{i,j}[1,\cdot,\cdot] \delta_j \\
\delta_i^\top \mathbf{H}_{M}^{i,j}[2,\cdot,\cdot] \delta_j \\
\vdots \\
\delta_i^\top \mathbf{H}_{M}^{i,j}[M,\cdot,\cdot] \delta_j
\end{array}
\right],
\end{equation}
where \( \delta_i \in \mathbb{R}^{n \times 1} \) is the error at location \( \hat{\mathbf{x}}_i \). The Jacobian and Hessian matrices are defined as $\mathbf{J}_{M}^{i} \coloneqq \left. \frac{\partial \mathbf{m}}{\partial \mathbf{z}_i} \right|_{\mathbf{Z} = \hat{\mathbf{X}}} \in \mathbb{R}^{M \times n},$ and $\mathbf{H}_{M}^{i,j} \coloneqq \left. \frac{\partial^2 \mathbf{m}}{\partial \mathbf{z}_i \partial \mathbf{z}_j^\top} \right|_{\mathbf{Z} = \hat{\mathbf{X}}} \in \mathbb{R}^{M \times n \times n}$,
respectively. Here, $\mathbf{H}_{M}^{i,j}[t,:,:] \in \mathbb{R}^{n \times n}$ denotes the Hessian matrix of the $t$-th component of $\mathbf{m}$ with respect to $\mathbf{z}_i$ and $\mathbf{z}_j$.

Similarly, the corrected covariance matrix $\mathcal{S} \coloneqq \mathbf{S}(\mathbf{X}) \in \mathbb{R}^{M \times M}$ is approximated using a second-order Taylor expansion around $\hat{\mathbf{X}}$ as
{\small
\begin{multline}\label{eq:cov_update}
\mathcal{S} \coloneqq \hat{\mathcal{S}}
+ \sum_{i=1}^{T} \mathbf{J}_{\mathcal{S}}^i\delta_i \\ + \frac{1}{2} \sum_{i=1}^{T} \sum_{j=1}^{T} 
\scalebox{0.85}{$
\left[
\begin{array}{ccc}
\delta_i^\top \mathbf{H}_{\mathcal{S}}^{i,j}[1,1,:,:] \delta_j & \cdots & \delta_i^\top \mathbf{H}_{\mathcal{S}}^{i,j}[1,M,:,:] \delta_j \\
\vdots & \ddots & \vdots \\
\delta_i^\top \mathbf{H}_{\mathcal{S}}^{i,j}[M,1,:,:] \delta_j & \cdots & \delta_i^\top \mathbf{H}_{\mathcal{S}}^{i,j}[M,M,:,:] \delta_j
\end{array}
\right]$},
\end{multline}
}where $\mathbf{J}_{\mathcal{S}}^{i} \coloneqq 
\left. \frac{\partial \mathbf{S}}{\partial \mathbf{z}_i} \right|_{\mathbf{Z} = \hat{\mathbf{X}}} \in \mathbb{R}^{M \times M \times n}$ denotes the Jacobian and $\mathbf{H}_{\mathcal{S}}^{i,j} \coloneqq 
\left. \frac{\partial^2 \mathbf{S}}{\partial \mathbf{z}_i \partial \mathbf{z}_j^\top} \right|_{\mathbf{Z} = \hat{\mathbf{X}}} \in \mathbb{R}^{M \times M \times n \times n}$ denotes the Hessian tensor of the covariance function. 

While the expressions for the corrected mean and covariance in \eqref{eq:mean_update} and \eqref{eq:cov_update} involve first and second order gradients with respect to each measurement location, the computational burden appears significant. However, a closer analysis of the structure of the Jacobian and Hessian tensors reveals that most entries are zero. This sparsity arises, as explained in the next section, because differentiating the kernel function with respect to a specific measurement location \( {\mathbf{z}}_k \) only influences the $k$-th column in $K_{e-\mathrm{T}}$ and $k$-th row and column in $K_{\mathrm{T}-\mathrm{T}}$. 

\subsection{Gradient of the Test-Train Kernel Matrix}
To find the gradient of the kernel matrix $K_{e-\mathrm{T}}$, we first compute the gradient of the kernel function \( k(\mathbf{x}, {\mathbf{z}}) \) with respect to its second argument, \( {\mathbf{z}} \). For the SE kernel defined in Equation~\eqref{eq:kernel}, this gradient $
\frac{\partial k(\mathbf{x}, {\mathbf{z}})}{\partial {\mathbf{z}}} \in \mathbb{R}^{n}$ is given by
\begin{equation}
\frac{\partial k(\mathbf{x}, \mathbf{z})}{\partial \mathbf{z}}
= \frac{1}{\beta^2} (\mathbf{x} - \mathbf{z}) \, k(\mathbf{x}, \mathbf{z}).
\end{equation}
Stacking the gradients for all \( M \) test points, the gradient of the test-train kernel matrix \(K_{e-\mathrm{T}} \) with respect to a single training input \( {\mathbf{z}}_k \) is a tensor in \( \mathbb{R}^{M \times T \times n} \). As only the \( k \)-th column of \(K_{e-\mathrm{T}} \) depends on \( {\mathbf{z}}_k \), the tensor is sparse, with non-zero entries defined as
\begin{equation}
\left[ \frac{\partial {K}_{e\text{-}\mathrm{T}}}{\partial \mathbf{z}_k} \right]_{i,j,:} \bigg|_{\mathbf{Z}=\hat{\mathbf{X}}}
=
\begin{cases}
\left( \frac{\partial k(\mathbf{x}_{e,i}, \mathbf{z})}{\partial \mathbf{z}} \big|_{\mathbf{z} = \hat{\mathbf{x}}_k} \right)^\top, & j = k, \\
\mathbf{0}_n^\top, & \text{otherwise},
\end{cases}
\end{equation}
where \( \left[ \cdot \right]_{i,j,:} \) denotes the \( 1 \times n \) vector at the \( (i,j) \) position of the tensor.

\subsection{Gradient of the Training Kernel Matrix}
For the training kernel matrix \( K_{\mathrm{T}-\mathrm{T}} \in \mathbb{R}^{T \times T} \), the gradient with respect to the training location \( {\mathbf{z}}_k \in \mathbb{R}^n \) affects both the \( k \)-th row and \( k \)-th column. The full gradient tensor is given by $\frac{\partial {K}_{\mathrm{T}\text{-}\mathrm{T}}}{\partial \mathbf{z}_k} \bigg|_{\mathbf{Z} = \hat{\mathbf{X}}} 
\in \mathbb{R}^{T \times T \times n}$, where each plane \( \left[ \frac{\partial K_{\mathrm{T}-\mathrm{T}}}{\partial \hat{\mathbf{x}}_k} \right]_{j,l,:} \in \mathbb{R}^n \) is defined as
\begin{equation}
\left[ \frac{\partial {K}_{\mathrm{T}\text{-}\mathrm{T}}}{\partial \mathbf{z}_k} \right]_{j,l,:}
\bigg|_{\mathbf{Z} = \hat{\mathbf{X}}}
=
\begin{cases}
\left( \frac{\partial k(\hat{\mathbf{x}}_j, \mathbf{z})}{\partial \mathbf{z}} \big|_{\mathbf{z} = \hat{\mathbf{x}}_k} \right)^\top, & l = k, \, j \neq k, \\
\left( \frac{\partial k(\mathbf{x}, \hat{\mathbf{x}}_l)}{\partial \mathbf{x}} \big|_{\mathbf{x} = \hat{\mathbf{x}}_k} \right)^\top, & j = k, \, l \neq k, \\
\mathbf{0}_n^\top, & \text{otherwise}.
\end{cases}
\end{equation}

\section{Gradient of the GP Mean and Covariance}

The Jacobians of the GP mean function $\mathbf{m}$ and covariance function $\mathbf{S}$ with respect to the training location $\mathbf{z}_i$, evaluated at $\hat{\mathbf{X}}$, define the terms $\mathbf{J}_M^i$ and $\mathbf{J}_{\mathcal{S}}^i$ in the Taylor expansion of the corrected moments in \eqref{eq:mean_update} and \eqref{eq:cov_update}. These Jacobians are simplified by exploiting the sparsity of the kernel matrix gradients, which are non-zero only in the $i$-th row or column.

\subsection{Gradient of the GP Mean Function}
The Jacobian of $\mathbf{m}$ as defined in \eqref{eq:mean_function_def} with respect to the $i$-th training location $\mathbf{z}_i$, evaluated at $\mathbf{Z} = \hat{\mathbf{X}}$, is $\mathbf{J}_M^i \in \mathbb{R}^{M \times n}$.
Let $\mathbf{K}(\mathbf{Z}) = \mathbf{K}_{\mathrm{T}-\mathrm{T}}(\mathbf{Z}) + \sigma_y^2 \mathbf{I}_\mathrm{T}$, and define $\mathbf{K} \coloneqq \mathbf{K}(\hat{\mathbf{X}})$ and $\mathbf{K}_{e-\mathrm{T}} \coloneqq \mathbf{K}_{e-\mathrm{T}}(\hat{\mathbf{X}})$ for brevity.

Applying the chain rule and the matrix inverse identity, $\frac{\partial \mathbf{K}(\mathbf{Z})^{-1}}{\partial \mathbf{z}_i} = -\mathbf{K}(\mathbf{Z})^{-1} \frac{\partial \mathbf{K}(\mathbf{Z})}{\partial \mathbf{z}_i} \mathbf{K}(\mathbf{Z})^{-1}$ yields
\begin{multline}
\mathbf{J}_M^i
= \frac{\partial \mathbf{m}}{\partial \mathbf{z}_i} \bigg|_{\mathbf{Z}=\hat{\mathbf{X}}} 
= \left( \frac{\partial \mathbf{K}_{e-\mathrm{T}}}{\partial \mathbf{z}_i} \bigg|_{\hat{\mathbf{X}}} \right) 
\mathbf{K}^{-1} \mathbf{Y}
\\-
\mathbf{K}_{e-\mathrm{T}} \mathbf{K}^{-1}
\left( \frac{\partial \mathbf{K}_{\mathrm{T}-\mathrm{T}}}{\partial \mathbf{z}_i} \bigg|_{\hat{\mathbf{X}}} \right) 
\mathbf{K}^{-1} \mathbf{Y}.
\end{multline}
Let $\mathbf{c} \coloneqq \mathbf{K}^{-1} \mathbf{Y} \in \mathbb{R}^{T \times 1}$ and $\mathbf{P} \coloneqq \mathbf{K}_{e-\mathrm{T}} \mathbf{K}^{-1} \in \mathbb{R}^{M \times T}$. Let $\mathbf{J}_{e-T}^i \coloneqq \frac{\partial \mathbf{K}_{e-\mathrm{T}}}{\partial \mathbf{z}_i}\big|_{\hat{\mathbf{X}}} \in \mathbb{R}^{M \times T \times n}$ and $\mathbf{J}_{T-T}^i \coloneqq \frac{\partial \mathbf{K}_{\mathrm{T}-\mathrm{T}}}{\partial \mathbf{z}_i}\big|_{\hat{\mathbf{X}}} \in \mathbb{R}^{T \times T \times n}$.
Using the sparsity of $\mathbf{J}_{e-T}^i$ (non-zero only at column $i$), the $\mathbf{J}_M^i$ Jacobian simplifies to
\begin{equation}\label{eq:mean_first_deri}
\mathbf{J}_M^i = \mathbf{J}_{e-T}^i[\cdot, i, \cdot] \mathbf{c}[i] - \mathbf{P} \mathbf{J}_{T-T}^i \mathbf{c},
\end{equation}
where $\mathbf{J}_{e-T}^i[\cdot, i, \cdot] \in \mathbb{R}^{M \times n}$ is the $i$-th slice of the kernel gradient tensor, $\mathbf{c}[i]$ is the $i$-th element of $\mathbf{c}$, and the second term denotes a full tensor contraction resulting in an $M \times n$ matrix \cite{muzaffar.SYNGE.SCHILD1949}.

\subsection{Gradient of the GP Covariance Function}
The Jacobian of $\mathbf{S}$ as defined in \eqref{eq:cov_function_def} with respect to $\mathbf{z}_i$, evaluated at $\mathbf{Z} = \hat{\mathbf{X}}$, is $\mathbf{J}_{\mathcal{S}}^i \in \mathbb{R}^{M \times M \times n}$. Since $K_{e-e}$ is constant, its gradient is zero. Let $\mathbf{A}(\mathbf{Z}) = \mathbf{K}_{e-\mathrm{T}}(\mathbf{Z}) \in \mathbb{R}^{M \times T}$ and $\mathbf{A} \coloneqq \mathbf{A}(\hat{\mathbf{X}})$. Applying the product rule for matrix gradients yields
\begin{multline} \label{eq:cov_jacobian_expanded}
\mathbf{J}_{\mathcal{S}}^i
= \frac{\partial \mathbf{S}}{\partial \mathbf{z}_i} \bigg|_{\mathbf{Z}=\hat{\mathbf{X}}} 
= - \left( \frac{\partial \mathbf{A}}{\partial \mathbf{z}_i} \right) \mathbf{K}^{-1} \mathbf{A}^{\top} 
\\- \mathbf{A} \left( \frac{\partial \mathbf{K}^{-1}}{\partial \mathbf{z}_i} \right) \mathbf{A}^{\top} 
- \mathbf{A} \mathbf{K}^{-1} \left( \frac{\partial \mathbf{A}^{\top}}{\partial \mathbf{z}_i} \right).
\end{multline}
Using the identity $\frac{\partial \mathbf{K}^{-1}}{\partial \mathbf{z}_i} = -\mathbf{K}^{-1} \frac{\partial \mathbf{K}}{\partial \mathbf{z}_i} \mathbf{K}^{-1}$, and noting that the first and third terms are transposes of each other, the Jacobian $\mathbf{J}_{\mathcal{S}}^i$ is simplified. Let $\mathbf{P} \coloneqq \mathbf{A} \mathbf{K}^{-1} \in \mathbb{R}^{M \times T}$, $\mathbf{J}_{\mathbf{A}}^i \coloneqq \frac{\partial \mathbf{A}}{\partial \mathbf{z}_i}\big|_{\hat{\mathbf{X}}} \in \mathbb{R}^{M \times T \times n}$ and $\mathbf{J}_{\mathbf{K}}^i \coloneqq \frac{\partial \mathbf{K}}{\partial \mathbf{z}_i}\big|_{\hat{\mathbf{X}}} \in \mathbb{R}^{T \times T \times n}$.
\begin{equation}\label{eq:cov_step_vector}
\mathbf{J}_{\mathcal{S}}^i
= -
\left( \mathbf{J}_{\mathbf{A}}^i \mathbf{K}^{-1} \mathbf{A}^{\top} \right) - \left( \mathbf{J}_{\mathbf{A}}^i \mathbf{K}^{-1} \mathbf{A}^{\top} \right)^{\top}
+ \mathbf{P} \mathbf{J}_{\mathbf{K}}^i \mathbf{P}^{\top},
\end{equation}
where $\mathbf{J}{\mathbf{A}}^i$ and $\mathbf{J}{\mathbf{K}}^i$ denote the kernel gradient tensors, which are sparse and have non-zero entries only in column $i$ for $\mathbf{J}{\mathbf{A}}^i$ and in row and column $i$ for $\mathbf{J}{\mathbf{K}}^i$. The term $\mathbf{P} \mathbf{J}_{\mathbf{K}}^i \mathbf{P}^{\top}$ represents a full tensor contraction, yielding an output tensor of dimension $M \times M \times n$.

\section{Offline Computation of GP Gradients for Efficient Online Updates}\label{sec:section4}
The analytical structure of the GP moments reveals a critical property: the dependence on the measurement vector \( \mathbf{Y} \) is separable from the dependence on the planned measurement locations \( \hat{\mathbf{X}} \). This separation is key to avoiding the \( \mathcal{O}(T^3) \) computational bottleneck of a full GP retraining when the measurement data set is corrected.

\subsection{Decomposition and Precomputation of Gradients}
The mean function $\mathbf{m}(\mathbf{Z})$ and its evaluation $\hat{\mathcal{M}} = \mathbf{K}_{e-\mathrm{T}} \mathbf{K}^{-1} \mathbf{Y}$, is linear with respect to the measurement vector \( \mathbf{Y} \). Consequently, all of its derivatives with respect to the training locations, including the Jacobian \( \mathbf{J}_M^i \) and Hessian \( \mathbf{H}_{M}^{i,j} \), are also linear in \( \mathbf{Y} \).

In contrast, the GP covariance function
as defined in \eqref{eq:cov_function_def}, is independent of \( \mathbf{Y} \). All of its derivatives, such as \( \mathbf{J}_{\mathcal{S}}^i \) and \( \mathbf{H}_{\mathcal{S}}^{i,j} \) dependent only on training locations \( \hat{\mathbf{X}} \), test locations \( \mathbf{X}_e \), and kernel hyperparameters.

Leveraging this fact, a two-stage framework is proposed. The computationally expensive, location-dependent structural components of the Jacobians and Hessians are precomputed offline, using automatic differentiation. The final derivatives can then be assembled online by linearly combining these precomputed tensors with the measurement vector \( \mathbf{Y} \) as needed, enabling real-time model correction.

\subsection{Mean Function Update}
The Jacobian \( \mathbf{J}_M^i \in \mathbb{R}^{M \times n} \) can be expressed as a tensor contraction $\mathbf{J}_M^i = \mathbf{F}^i(\hat{\mathbf{X}}, \mathbf{X}_e) \mathbf{Y}$, where \( \mathbf{F}^i \in \mathbb{R}^{M \times n \times T} \) is the \( \mathbf{Y} \)-independent structural tensor that is precomputed offline. Similarly, the Hessian \( \mathbf{H}_M^{i,j} \in \mathbb{R}^{M \times n \times n} \) is given by $\mathbf{H}_M^{i,j} = \mathbf{G}^{i,j}(\hat{\mathbf{X}}, \mathbf{X}_e) \mathbf{Y}$, where \( \mathbf{G}^{i,j} \in \mathbb{R}^{M \times n \times n \times T} \) is the corresponding precomputed fourth-order tensor.

By substituting these tensors into the original Taylor series expansion in \eqref{eq:mean_update} and \eqref{eq:cov_update}, we get
\begin{multline}
\mathcal{M} \coloneqq \hat{\mathcal{M}} + \sum_{i=1}^{T} \left( \mathbf{F}^i \mathbf{Y} \right) \delta_i \\
+ \frac{1}{2} \sum_{i=1}^{T} \sum_{j=1}^{T}
\scalebox{0.95}{$\begin{bmatrix}
\delta_i^\top \left( \mathbf{G}^{i,j} \mathbf{Y} \right)[1,\cdot,\cdot] \delta_j \\
\delta_i^\top \left( \mathbf{G}^{i,j} \mathbf{Y} \right)[2,\cdot,\cdot] \delta_j \\
\vdots \\
\delta_i^\top \left( \mathbf{G}^{i,j} \mathbf{Y} \right)[M,\cdot,\cdot] \delta_j
\end{bmatrix}$}.
\end{multline}
As a result, the online computation only involves tensor contractions between the precomputed structural components (\( \mathbf{F}^i \), \( \mathbf{G}^{i,j} \)) and the measurement vector \( \mathbf{Y} \) and the perturbation vectors $\{\delta_i\}_{i=1}^T$.

\subsection{Covariance Function Update}
An analogous update is performed for the covariance function \( \mathbf{S} \), whose value and derivatives, including the Jacobian \( \mathbf{J}_{\mathcal{S}}^i \in \mathbb{R}^{M \times M \times n} \) and the cross-Hessian \( \mathbf{H}_{\mathcal{S}}^{i,j} \in \mathbb{R}^{M \times M \times n \times n} \) are independent of the measurement vector \( \mathbf{Y} \). These tensors depend solely on the planned measurement locations \( \hat{\mathbf{X}} \) and the test locations \( \mathbf{X}_e \), and can be precomputed and stored offline.

The online correction of the covariance matrix \( \mathcal{S} \) is efficiently performed by assembling these precomputed tensors with the localization errors \( \{ \delta_i \}_{i=1}^T \). Using a second-order Taylor expansion, the corrected covariance is given by
\begin{equation}
\mathcal{S} \coloneqq \hat{\mathcal{S}} + \sum_{i=1}^{T} \mathbf{J}_{\mathcal{S}}^i \delta_i 
+ \frac{1}{2} \sum_{i=1}^{T} \sum_{j=1}^{T} \left( \delta_i^\top \mathbf{H}_{\mathcal{S}}^{i,j} \delta_j \right),
\end{equation}
where \( \hat{\mathcal{S}} \) is the covariance computed using the planned locations \( \hat{\mathbf{X}} \). The term \( \mathbf{P} \mathbf{J}_{\mathbf{K}}^i \mathbf{P}^\top \) denotes a full tensor contraction, resulting in an output tensor of dimension \( M \times M \times n \),
where the terms \( \mathbf{J}_{\mathcal{S}}^i \delta_i \) and \( \delta_i^\top \mathbf{H}_{\mathcal{S}}^{i,j} \delta_j \) denote the tensor contractions required to produce the \( M \times M \) correction matrices.

By precomputing all structural components (\( \mathbf{F}^i, \mathbf{G}^{i,j}, \mathbf{J}_{\mathcal{S}}^i, \mathbf{H}_{\mathcal{S}}^{i,j} \)), online evaluation of both \( \mathcal{M} \) and \( \mathcal{S} \) is reduced to a series of tensor products, bypassing full GP recomputation.

\subsection{Computational Complexity Analysis}
A complete GP retraining, which serves as the baseline for comparison, has a computational cost dominated by the inversion of the $T \times T$ covariance matrix, requiring $\mathcal{O}(T^3)$ operations. The developed correction framework avoids this bottleneck by precomputing the Jacobian and Hessian tensors offline. 

In the online phase, the first-order mean correction has complexity $\mathcal{O}(T M n)$, while the second-order mean correction involving the Hessian term is $\mathcal{O}(T^2 M n^2)$. For the covariance update, the first-order term costs $\mathcal{O}(T M^2 n)$, and the second-order term costs $\mathcal{O}(T^2 M^2 n^2)$. Therefore, the total worst-case cost is $\mathcal{O}(T^2 M^2 n^2)$. Since the GP derivative tensors are highly sparse, the actual computational cost is expected to be substantially lower.

\begin{remark}
Note that the above complexity represents the case when all $T$ measurement locations are corrected simultaneously. In practical scenarios such as robotic SLAM or loop closure, only a small subset $K \ll T$ of measurement locations require correction. Thus, the online complexity reduces to $\mathcal{O}(K M^2 n + K^2 M^2 n^2)$, which is more efficient than full retraining or the worst-case $\mathcal{O}(T^2 M^2 n^2)$.
\end{remark}

\begin{algorithm}
\caption{GP Moment Correction under Deterministic Measurement Location Errors}
\label{alg:gp_update_algorithm_corrected}
\begin{algorithmic}[1]
\State \textbf{Input}: Planned locations $\mathbf{\hat{X}} \in \mathbb{R}^{T \times n}$, Actual measurements $\mathbf{Y} \in \mathbb{R}^{T \times 1}$, Test locations $\mathbf{X}_e \in \mathbb{R}^{M \times n}$, GP hyperparameters $\alpha, \beta$
\State \textbf{Offline Phase}:
\State \quad Compute and store gradient tensors
\State \quad Compute initial kernel matrices: $K_{\mathrm{T}-\mathrm{T}}$, $K_{e-\mathrm{T}}$, $K_{e-e}$
\State \quad Compute gradient operators: $\{\mathbf{F}^i\}_{i=1}^T$, $\{\mathbf{G}^{i,j}\}_{i,j=1}^T$, $\{\mathbf{J}_{\mathcal{S}}^i\}_{i=1}^T$, $\{\mathbf{H}_{\mathcal{S}}^{i,j}\}_{i,j=1}^T$

\State \textbf{Online Phase}:
\State \textbf{Input}: Full set of location perturbations $\{{\delta}_i\}_{i=1}^T$
\State \quad Compute initial GP moments: $\hat{\mathcal{M}}$, $\hat{\mathcal{S}}$
\State \quad Instantiate $\mathbf{Y}$-dependent Jacobians/Hessians using stored operators $\mathbf{F}^i, \mathbf{G}^{i,j}$ and current $\mathbf{Y}$.
\State \quad $\mathbf{J}_{M}^i \leftarrow \mathbf{F}^i \mathbf{Y}$ (for all $i=1,\dots,T$)
\State \quad $\mathbf{H}_{M}^{i,j} \leftarrow \mathbf{G}^{i,j} \mathbf{Y}$ (for all $i,j=1,\dots,T$)
\State \quad Update $\mathcal{M}$ using $\mathbf{J}_{M}^i, \mathbf{H}_{M}^{i,j}$ and $\{{\delta}_i\}_{i=1}^T$ (Eq. \ref{eq:mean_update})
\State \quad Update $\mathcal{S}$ using $\mathbf{J}_{\mathcal{S}}^i, \mathbf{H}_{\mathcal{S}}^{i,j}$ and $\{{\delta}_i\}_{i=1}^T$ (Eq. \ref{eq:cov_update})
\State \textbf{Output}: $\mathcal{M}$, $\mathcal{S}$
\end{algorithmic}
\end{algorithm}

\section{Convergence Analysis}
In the preceding sections, the GP mean $\mathbf{m}$  was defined as a function of a set of training locations $\mathbf{Z} = \{\mathbf{z}_i\}_{i=1}^T$. For the convergence analysis presented in this section, the mean function is treated as a function that operates on a single stacked vector in a high-dimensional Euclidean space. To facilitate the analysis, let $\boldsymbol{z} \in \mathbb{R}^{nT}$ be the stacked vector of all training locations, defined as $\boldsymbol{z} = \text{vec}(\mathbf{z}_1, \dots, \mathbf{z}_T)$. The domain of all possible training locations is assumed to lie in the compact set $\mathcal{D} \subset \mathbb{R}^{nT}$. Also, let stacked vector of planned locations be $\hat{\boldsymbol{x}} = [\hat{\mathbf{x}}_1, \dots, \hat{\mathbf{x}}_T] \in \mathbb{R}^{nT}$, and the stacked vector of actual locations is $\boldsymbol{x} = [\mathbf{x}_1, \dots, \mathbf{x}_T] \in \mathbb{R}^{nT}$. The total perturbation is then the single vector $\boldsymbol{\delta} = \boldsymbol{x} - \hat{\boldsymbol{x}} \in \mathbb{R}^{nT}$.

The following theorem demonstrates that the mean function is analytic, which is the necessary condition for a Taylor series to converge.

\begin{theorem}\label{thm:TaylorSeries}
Given the GP mean function $\mathbf{m}: \mathcal{D} \to \mathbb{R}^M$, if the kernel function $k$ is analytic on the domain $\mathcal{X} \subset \mathbb{R}^n$, then the mean function $\mathbf{m}$ is also analytic on $\mathcal{D}$. The true mean value $\mathcal{M} \coloneqq \mathbf{m}(\boldsymbol{x})$ can be represented as a multivariate Taylor series expansion around the point $\hat{\boldsymbol{x}}$, with perturbation $\boldsymbol{\delta} = \boldsymbol{x} - \hat{\boldsymbol{x}}$, as
\begin{equation}
 \mathcal{M} = \sum_{N=0}^{\infty} \frac{1}{N!} \nabla^N \mathbf{m}(\hat{\boldsymbol{x}}) \,[\boldsymbol{\delta}^N],
\end{equation}
where $\nabla^N \mathbf{m}(\hat{\boldsymbol{x}})$ is the $N$-th order gradient tensor of $\mathbf{m}$ at $\hat{\boldsymbol{x}}$, and $[\boldsymbol{\delta}^N]$ denotes the $N$-fold tensor application of the perturbation vector $\boldsymbol{\delta}$. Furthermore, the Taylor series converges uniformly to $\mathbf{m}$ on any compact subset of $\mathcal{D}$.
\end{theorem}

\begin{proof}
The GP mean function $\mathbf{m}(\boldsymbol{z})$ is defined by the expression in \eqref{eq:mean_function_def}. Since $k$ is analytic on $\mathcal{X}$, each entry in the kernel matrices $\mathbf{K}_{e-\mathrm{T}}$ and $\mathbf{K}_{\mathrm{T}-\mathrm{T}}$ is an analytic function of the components of $\boldsymbol{z}$. The operations of matrix addition, matrix multiplication, and matrix inversion (which is a rational function of the matrix entries) all preserve analyticity. Consequently, the composite function $\mathbf{m}(\boldsymbol{z})$ is real analytic on its domain $\mathcal{D}$.

By the definition of a real analytic function, $\mathbf{m}(\boldsymbol{z})$ can be represented by its Taylor series expansion around any point $\hat{\boldsymbol{x}} \in \mathcal{D}$. For any compact subset $\mathcal{V} \subset \mathcal{D}$, the perturbation $\boldsymbol{\delta} = \boldsymbol{x} - \hat{\boldsymbol{x}}$ is uniformly bounded, i.e., $\|\boldsymbol{x} - \hat{\boldsymbol{x}}\| \leq L_f$ for all $\boldsymbol{x} \in \mathcal{V}$. By the multivariate Taylor theorem, the remainder term $R_N(\boldsymbol{x}) = \mathbf{m}(\boldsymbol{x}) - T_N(\boldsymbol{x})$, where $T_N$ is the $N$-th order Taylor polynomial, satisfies
\begin{equation}
 \|R_N(\boldsymbol{x})\| \leq \frac{M_{N+1} \, L_f^{N+1}}{(N+1)!},
\end{equation}
where $M_{N+1} > 0$ bounds the $(N+1)$-th order gradient tensor on $\mathcal{V}$. Since $(N+1)!$ grows faster than $L_f^{N+1}$, the remainder $\|R_N(\boldsymbol{x})\| \to 0$ as $N \to \infty$ (see \cite[Theorem~5.15]{muzaffar.Rudin1976}). 
\end{proof}

Theorem~\ref{thm:TaylorSeries} proves that the mean function $\mathbf{m}$ is infinitely differentiable and can be represented by a convergent Taylor series. Note that Assumption \ref{assump:2} implies a bound on the total perturbation vector $\boldsymbol{\delta} = \boldsymbol{x} - \hat{\boldsymbol{x}}$. Specifically, if $\|\delta_i\| \leq \delta_{\max}$ for all $i=1,\dots,T$, then the $L_2$-norm of the stacked vector $\boldsymbol{\delta}$ is bounded by
\begin{equation}\label{eq:deltaBound}
\|\boldsymbol{\delta}\| \leq \sqrt{\sum_{i=1}^T (\delta_{\max})^2} = \delta_{\max} \sqrt{T}.
\end{equation}
 Building on Theorem~\ref{thm:TaylorSeries} and using the bound in \eqref{eq:deltaBound}, the following theorem presents bounds on the required number of higher-order gradients
to achieve the desired approximation accuracy within a
compact subset of the input domain.

\begin{theorem}\label{thm:GradientCount}
Let $\mathcal{V} \subset \mathcal{D}$ be compact, with $\| \boldsymbol{x} - \hat{\boldsymbol{x}} \| \leq \beta$ for all $\boldsymbol{x} \in \mathcal{V}$, where $\beta \coloneqq \delta_{\max}\sqrt{T}$. Given a desired accuracy $\epsilon > 0$, the minimum order $N$ of the Taylor polynomial $T_N(\boldsymbol{x})$ required to guarantee
\begin{equation}
 \| \mathbf{m}(\boldsymbol{x}) - T_N(\boldsymbol{x}) \| \leq \epsilon, \quad \forall \boldsymbol{x} \in \mathcal{V},
\end{equation}
is the smallest integer $N$ satisfying $\frac{M_{N+1} \, \beta^{N+1}}{(N+1)!} \leq \epsilon$, where $M_{N+1} \ge \sup_{\boldsymbol{\kappa} \in \mathcal{V}} \| \nabla^{N+1} \mathbf{m}(\boldsymbol{\kappa}) \|$ bounds the $(N+1)$-th order gradient tensor.
\end{theorem}

\begin{proof}
From the proof of Theorem~\ref{thm:TaylorSeries}, the remainder $R_N(\boldsymbol{x}) = \mathbf{m}(\boldsymbol{x}) - T_N(\boldsymbol{x})$ is bounded by
$\|R_N(\boldsymbol{x})\| \leq \frac{M_{N+1} \, \|\boldsymbol{x} - \hat{\boldsymbol{x}}\|^{N+1}}{(N+1)!}$. Using the given bound $\|\boldsymbol{x} - \hat{\boldsymbol{x}}\| \leq \beta$, we have
\begin{equation}
 \|R_N(\boldsymbol{x})\| \leq \frac{M_{N+1} \, \beta^{N+1}}{(N+1)!}.
\end{equation}
To satisfy the desired accuracy $\epsilon$, we must choose $N$ large enough such that the remainder bound is less than or equal to $\epsilon$. The minimum $N$ is therefore the smallest integer for which this inequality holds.
\end{proof}

\begin{remark}
Theorem~\ref{thm:GradientCount} formalizes the trade-off between the desired approximation accuracy $\epsilon$ and the Taylor expansion order $N$. Due to the factorial in the denominator, the error bound decreases rapidly with $N$. For small perturbations (i.e., $\beta < 1$), low-order expansions (e.g., $N=2$) often suffice in practice.
\end{remark}

\section{Simulation Results}
Two simulations are conducted to validate the developed method for refining GP models trained on corrupted datasets. The first simulation, performed in a 1D domain, introduces spatially varying errors to measurement locations, simulating random, spatially dependent noise. The second simulation, conducted in a 2D domain, applies a constant offset \( {\delta} \) to all locations, representing a uniform sensor bias across the entire field.

\subsection{1-Dimensional Example}
This simulation considers the function \( f_1(\mathbf{x}) = 2+\sin(2\pi \mathbf{x}) \), with \( \mathbf{x} \in [0,1] \). Measurement locations are uniformly spaced as \( \hat{\mathbf{X}} = \{ 0, 0.1, 0.2, \dots, 1 \} \). A SE kernel is used with hyperparameters \( \alpha = 1 \), \( \beta = 0.1 \). Jacobians and Hessians from Equations~\eqref{eq:mean_first_deri} and~\eqref{eq:cov_step_vector} are computed using the CasADi symbolic framework \cite{Andersson2018}.

GP predictions are evaluated at uniformly distributed query points with $M=100$. A baseline GP model is trained on the original locations and values. Gaussian noise \( \epsilon \sim \mathcal{N}(0, 0.01^2) \) is added to simulate corrupted measurement locations. The corrupted GP model is shown in Figure~\ref{fig:case2_1D}, and the corrected model using gradient-based updates is shown in Figure~\ref{fig:case3_1D}. The error \( \|e\|_{(\cdot)} \) presented in Table~\ref{tab:percentage_improvement_1d} is the \( L_2 \) norm of the error vector computed over all query points \( X_e \) using
\[
\|e\|_{\mathcal{M}} = \sqrt{\sum_{i=1}^{M} \big( f_1(\mathbf{x_i}) - \mathcal{M}_i\big)^2},
\]
where \( \|e\|_{\hat{\mathcal{M}}} \) is computed similarly using corrupted GP model $\mathcal{\hat{M}}$. The absolute value of errors between the corrupted and perfect GP model, and between the corrected and perfect GP model, are compared in Figure~\ref{fig:error_comparison}. A comparison of the aggregate error norms and computational time, averaged over 100 simulations, is summarized in Table~\ref{tab:percentage_improvement_1d} and Table~\ref{tab:time_improvement_2d}, respectively.

\subsection{2-Dimensional Example}
The second simulation considers a scalar field defined on a bounded domain \( \mathcal{X} \subseteq \mathbb{R}^2 \) as
\[
f_2(\mathbf{x}) = \sin(2\pi x) \cos(2\pi y), \quad \mathbf{x} = (x, y) \in [0,1]^2.
\]
The scalar field is visualized in Figure~\ref{fig:D Surface Plot of the GP predicted}. To simulate a uniform sensor bias, the \( x \)-coordinates of all measurement locations are perturbed by a fixed offset \( \delta_i = 0.1 \) as shown in Figure~\ref{fig:D Surface Plot with measurement locations}. A GP model is then trained on these corrupted locations using the true scalar field values. The corrected GP predictions are computed using the gradient-based update mechanism. Error trajectories for the corrupted and corrected mean functions are shown in Figure~\ref{fig:error_trajectories_2D}, demonstrating error reduction after applying the correction.

\pgfplotstableread[col sep=comma]{data/gp_summary_1D.csv}\dataOneD
\pgfplotstableread[col sep=comma]{data/gp_summary_2D.csv}\dataTwoD

\pgfplotsset{
    /pgf/number format/fixed, 
    /pgf/number format/precision=2, 
    /pgf/number format/zerofill,
}

\begin{table}
    \centering
    \caption{Average Error Norm ($||e||$) and Resulting Improvement Percentage from 100 Simulations.}
    \footnotesize
    \resizebox{\columnwidth}{!}{
        \begin{tabular}{lcc}
            \toprule
            \textbf{Improvement} & \textbf{1D Sim} & \textbf{2D Sim} \\
            \midrule
            
            $||e||_{\mathcal{\hat{M}}}$ & 
            \pgfplotstablegetelem{0}{norm_corrupted}\of\dataOneD 
            \pgfmathprintnumber{\pgfplotsretval} & 
            \pgfplotstablegetelem{0}{norm_corrupted}\of\dataTwoD
            \pgfmathprintnumber{\pgfplotsretval} \\ 

            $||e||_{\mathcal{{M}}}$ & 
            \pgfplotstablegetelem{0}{norm_corrected}\of\dataOneD
            \pgfmathprintnumber{\pgfplotsretval} &
            \pgfplotstablegetelem{0}{norm_corrected}\of\dataTwoD
            \pgfmathprintnumber{\pgfplotsretval} \\
            
            Improvement Percentage & 
            \pgfplotstablegetelem{0}{improvement_pct}\of\dataOneD
            \pgfmathprintnumber[precision=2]{\pgfplotsretval} $\%$ & 
            \pgfplotstablegetelem{0}{improvement_pct}\of\dataTwoD
            \pgfmathprintnumber[precision=2]{\pgfplotsretval} $\%$ \\
            
            \bottomrule
        \end{tabular}
    }
    \label{tab:percentage_improvement_1d}
\end{table}

\begin{table}
    \centering
    \caption{Average Computation Time for Correction vs. Full GP Training (100 Simulations).}
    \footnotesize
    \resizebox{\columnwidth}{!}{
        \begin{tabular}{lcc}
            \toprule
            \textbf{Method} & \textbf{Time (secs) 1D} & \textbf{Time (secs) 2D} \\
            \midrule
            
            Full Retraining & 
            \pgfplotstablegetelem{0}{time_retrain}\of\dataOneD
            \pgfmathprintnumber[sci, precision=2]{\pgfplotsretval} &
            \pgfplotstablegetelem{0}{time_retrain}\of\dataTwoD
            \pgfmathprintnumber[sci, precision=2]{\pgfplotsretval} \\

            Gradient Correction &
            \pgfplotstablegetelem{0}{time_correction_mult}\of\dataOneD
            \pgfmathprintnumber[sci, precision=2]{\pgfplotsretval} &
            \pgfplotstablegetelem{0}{time_correction_mult}\of\dataTwoD
            \pgfmathprintnumber[sci, precision=2]{\pgfplotsretval} \\
            
            \bottomrule
        \end{tabular}
    }
    \label{tab:time_improvement_2d}
\end{table}

\begin{figure}
    \centering
    \begin{tikzpicture}
        \begin{axis}[
            width=0.45\textwidth,
            height=0.35\textwidth,
            xlabel={$\mathbf{x}$},
            ylabel={$f_1(\mathbf{x})$},
            grid=both,
            xmin=0,
            xmax=1, 
        ]
            \addplot[
                color=blue,
                mark=none,
                thick,
            ] table [x=X_test, y=y_true, col sep=comma] {data/gp_simulation_results.csv};
            
            \addplot[
                color=red,
                mark=none,
                thick,
            ] table [x=X_test, y=y_pred_case1, col sep=comma] {data/gp_simulation_results.csv};
            
            \addplot [
                name path=upper_case1,
                draw=none,
            ] table [x=X_test, y expr=\thisrow{y_pred_case1} + 2*\thisrow{y_pred_std_case1}, col sep=comma] {data/gp_simulation_results.csv};
            
            \addplot [
                name path=lower_case1,
                draw=none,
            ] table [x=X_test, y expr=\thisrow{y_pred_case1} - 2*\thisrow{y_pred_std_case1}, col sep=comma] {data/gp_simulation_results.csv};
            
            \addplot [
                fill=purple!20,
                opacity=0.5,
            ] fill between [of=upper_case1 and lower_case1];
            
            \addplot[
                only marks,
                mark=x,
                black,
                mark size=3pt,
            ] table [x=Train_X_original, y=Train_Y_original, col sep=comma] {data/gp_simulation_results.csv};
            
            \node[anchor=north east, draw, fill=white, inner sep=1pt, font=\scriptsize] at (rel axis cs: 0.98, 0.98) {
                \begin{tabular}{@{}l@{}}
                    \textcolor{blue}{\rule{8pt}{1.5pt}} \hspace{2pt} True Function \\
                    \textcolor{red}{\rule{8pt}{1.5pt}} \hspace{2pt} Predicted Mean\\
                    \textcolor{purple!50}{\rule[1.5pt]{8pt}{3.5pt}} \hspace{2pt} 2$\sigma$ Region \\
                    \textcolor{black}{\(\times\)} \hspace{2pt} Training Points \\
                \end{tabular}
            };
        \end{axis}
    \end{tikzpicture}
    \caption{Baseline GP model prediction trained on the original, non-corrupted 1D measurement locations. The true function is shown in blue.}
    \label{fig:case1_1D}
\end{figure}
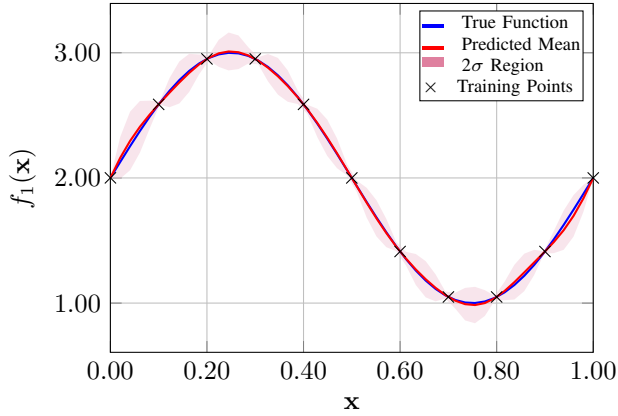

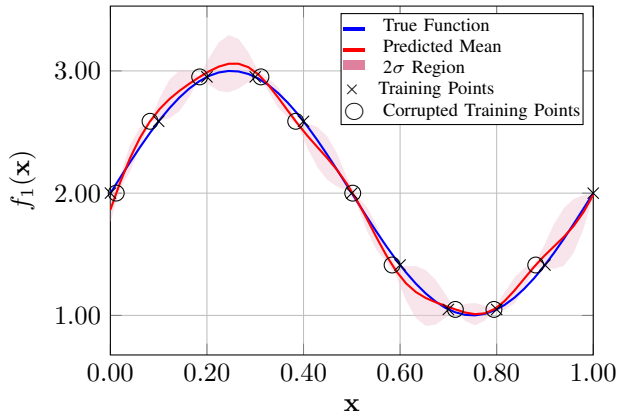
\begin{figure}
    \centering
    \begin{tikzpicture}
        \begin{axis}[
            width=0.45\textwidth,
            height=0.35\textwidth,
            xlabel={$\mathbf{x}$},
            ylabel={$ f_1(\mathbf{x})$},
            legend pos=north east, 
            legend style={font=\small},
            grid=both,
                        xmin=0,
            xmax=1, 
        ]
            \addplot[
                color=blue,
                mark=none,
                thick,
            ] table [x=X_test, y=y_true, col sep=comma] {data/gp_simulation_results.csv};
            
            \addplot[
                color=red,
                mark=none,
                thick,
            ] table [x=X_test, y=y_pred_case2, col sep=comma] {data/gp_simulation_results.csv};
            
            \addplot [
                name path=upper_case2,
                draw=none,
            ] table [x=X_test, y expr=\thisrow{y_pred_case2} + 2*\thisrow{y_pred_std_case2}, col sep=comma] {data/gp_simulation_results.csv};
            
            \addplot [
                name path=lower_case2,
                draw=none,
            ] table [x=X_test, y expr=\thisrow{y_pred_case2} - 2*\thisrow{y_pred_std_case2}, col sep=comma] {data/gp_simulation_results.csv};
            
            \addplot [
                fill=purple!20,
                opacity=0.5,
                forget plot,
            ] fill between [of=upper_case2 and lower_case2];
            
            \addplot[
                only marks,
                mark=x,
                black,
                mark size=3pt,
            ] table [x=Train_X_original, y=Train_Y_corrupted, col sep=comma] {data/gp_simulation_results.csv};

            \addplot[
                only marks,
                mark=o,
                black,
                mark size=3pt,
            ] table [x=Train_X_corrupted, y=Train_Y_corrupted, col sep=comma] {data/gp_simulation_results.csv};
            \node[anchor=north east, draw, fill=white, inner sep=1pt, font=\scriptsize] at (rel axis cs: 0.98, 0.98) {
            \begin{tabular}{@{}l@{}}
                    \textcolor{blue}{\rule{8pt}{1.5pt}} \hspace{2pt} True Function \\
                    \textcolor{red}{\rule{8pt}{1.5pt}} \hspace{2pt} Predicted Mean\\
                    \textcolor{purple!50}{\rule[1.5pt]{8pt}{3.5pt}} \hspace{2pt} 2$\sigma$ Region \\
                    \textcolor{black}{\(\times\)} \hspace{2pt} Training Points \\
                    \textcolor{black}{$\bigcirc$} \hspace{2pt} Corrupted Training Points
                \end{tabular}
            };
        \end{axis}
    \end{tikzpicture}
    \caption{GP model prediction trained on corrupted 1D measurement locations, demonstrating the model shift relative to the true function and original training points.}
    \label{fig:case2_1D}
\end{figure}
\begin{figure}
    \centering
    \begin{tikzpicture}
        \begin{axis}[
            width=0.45\textwidth,
            height=0.35\textwidth,
            xlabel={$\mathbf{x}$},
            ylabel={$ f_1(\mathbf{x})$},
            legend pos=north east, 
            legend style={font=\small},
            grid=both,
                        xmin=0,
            xmax=1, 
        ]
            \addplot[
                color=blue,
                mark=none,
                thick,
            ] table [x=X_test, y=y_true, col sep=comma] {data/gp_simulation_results.csv};
                        
            \addplot[
                color=red,
                mark=none,
                thick,
            ] table [x=X_test, y=y_pred_case3, col sep=comma] {data/gp_simulation_results.csv};
            
            \addplot [
                name path=upper_case3,
                draw=none,
            ] table [x=X_test, y expr=\thisrow{y_pred_case3} + 2*\thisrow{y_pred_std_case3}, col sep=comma] {data/gp_simulation_results.csv};
            
            \addplot [
                name path=lower_case3,
                draw=none,
            ] table [x=X_test, y expr=\thisrow{y_pred_case3} - 2*\thisrow{y_pred_std_case3}, col sep=comma] {data/gp_simulation_results.csv};
            
            \addplot [
                fill=purple!20,
                opacity=0.5,
                forget plot,
            ] fill between [of=upper_case3 and lower_case3];
            
            \addplot[
                only marks,
                mark=x,
                black,
                mark size=3pt,
            ] table [x=Train_X_original, y=Train_Y_original, col sep=comma] {data/gp_simulation_results.csv};
                        
            \node[anchor=north east, draw, fill=white, inner sep=1pt, font=\scriptsize] at (rel axis cs: 0.98, 0.98) {
                \begin{tabular}{@{}l@{}}
                    \textcolor{blue}{\rule{8pt}{1.5pt}} \hspace{2pt} True Function \\
                    \textcolor{red}{\rule{8pt}{1.5pt}} \hspace{2pt} Corrected Mean\\
                    \textcolor{purple!50}{\rule[1.5pt]{8pt}{3.5pt}} \hspace{2pt} 2$\sigma$ Region \\
                \end{tabular}
            };
        \end{axis}
    \end{tikzpicture}
    \caption{Corrected GP prediction for the 1D simulation. The model, updated using the gradient-based correction, shows the mean function (red line) and $2\sigma$ confidence interval realigned with the true function (blue line).}
    \label{fig:case3_1D}
\end{figure}
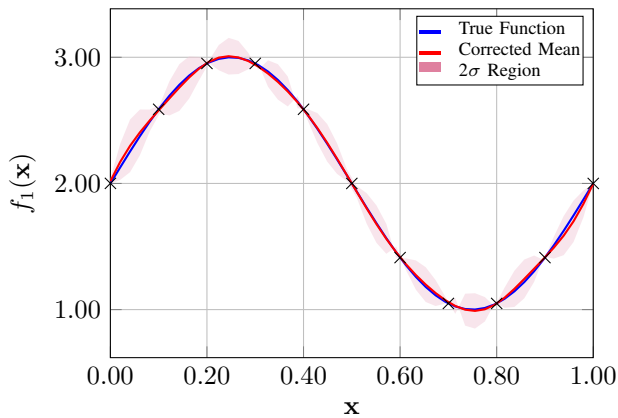
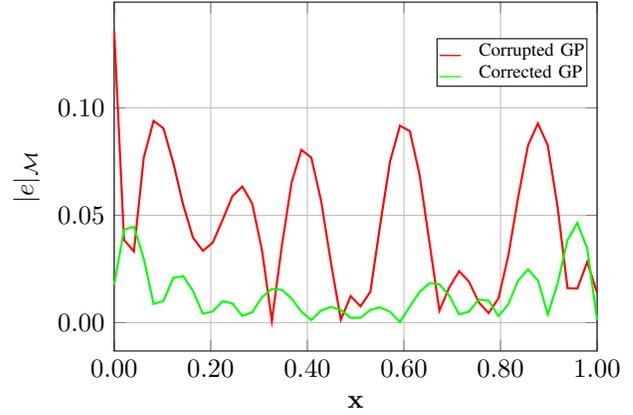
\begin{figure}
    \centering
    \begin{tikzpicture}
        \begin{axis}[
           width=0.45\textwidth,
            height=0.35\textwidth,
            xlabel={$\mathbf{x}$},
            ylabel={$|e|_{\mathcal{M}}$},
            legend pos=north east, legend style={font=\small},
            grid=both,
            xmin=0,
            xmax=1, 
            ]
            \addplot[
                color=red,
                thick,
            ] table [x=X_test, y=error_case2, col sep=comma] {data/gp_simulation_results.csv};
                        
            \addplot[
                color=green,
                thick,
            ] table [x=X_test, y=error_case3, col sep=comma] {data/gp_simulation_results.csv};
        \end{axis}
                    \node[anchor=north east, draw, fill=white, inner sep=1pt, font=\scriptsize] at (rel axis cs: 0.98, 0.98) {
                \begin{tabular}{@{}l@{}}
                    \textcolor{red}{\rule{8pt}{0.5pt}} \hspace{2pt} Corrupted GP \\
                    \textcolor{green}{\rule{8pt}{0.5pt}} \hspace{2pt} Corrected GP 
                \end{tabular}
            };
    \end{tikzpicture}
    \caption{Absolute error of the mean for the 1D simulation, comparing the corrupted GP against the corrected GP across the test domain.}
    \label{fig:error_comparison}
\end{figure}

\begin{figure}
        \centering
        \begin{tikzpicture}
            \begin{axis}[
                colormap/viridis,
                xlabel={${x}$},
                ylabel={$y$},
                zlabel={$ f_2(x,y)$},
                xticklabel style={font=\footnotesize},
                yticklabel style={font=\footnotesize},
                zticklabel style={font=\footnotesize},
                xlabel style={font=\footnotesize},
                ylabel style={font=\footnotesize},
                zlabel style={font=\footnotesize},
                height=0.8\columnwidth, 
                width=0.8\columnwidth,
                view={-45}{25}, 
                grid, 
            ]
            \addplot3[
                surf,
                mesh/rows=20,
                mesh/cols=20,
                shader=interp 
                ]
            table[x=X_test_x, y=X_test_y, z=y_true, col sep=comma] {data/gp_simulation_results2.csv}; 
            \end{axis}
        \end{tikzpicture}
        \caption{Visualization of the true scalar field for the 2D simulation.}
        \label{fig:D Surface Plot of the GP predicted}
\end{figure}
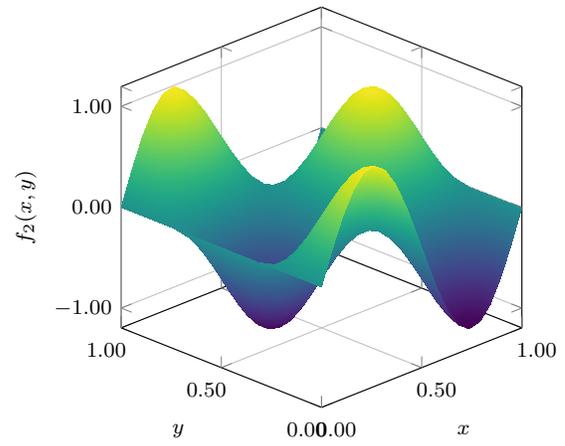
\begin{figure}
    \centering
    \begin{tikzpicture}
        \begin{axis}[%
            colormap/viridis,
            xlabel={$x$},
            ylabel={$y$},
            zlabel={$g(\mathbf{x})$},
            xticklabel style={font=\footnotesize},
            yticklabel style={font=\footnotesize},
            zticklabel style={font=\footnotesize},
            xlabel style={font=\footnotesize},
            ylabel style={font=\footnotesize},
            zlabel style={font=\footnotesize},
            height=0.90\columnwidth, 
            width=0.90\columnwidth,
            view={0}{90}, 
            grid, 
        ]
        \addplot3[
            surf,
            mesh/rows=20,
            mesh/cols=20,
            shader=interp 
            ]
        table[x=X_test_x, y=X_test_y, z=y_true, col sep=comma] {data/gp_simulation_results2.csv}; 
        \end{axis}

        \begin{axis}[%
            axis lines=none, 
            view={90}{90}, 
            width=.90\columnwidth, 
            height=.90\columnwidth, 
            grid=none, 
        ]
            \addplot[
                only marks,
                mark=*,
                color=black,
                mark size=2pt
                ]
            table[x=Train_X_x_original, y=Train_X_y_original, col sep=comma] {data/gp_simulation_results2.csv}; 

            \addplot[
                only marks,
                mark=*,
                color=red,
                mark size=2pt
                ]
            table[x=Train_X_x_corrupted, y=Train_X_y_corrupted, col sep=comma] {data/gp_simulation_results2.csv}; 
            \end{axis}
    \end{tikzpicture}
    \caption{Top-down view of the 2D scalar field, showing the true measurement locations (black dots) and the corrupted locations (red dots).}
    \label{fig:D Surface Plot with measurement locations}
\end{figure}
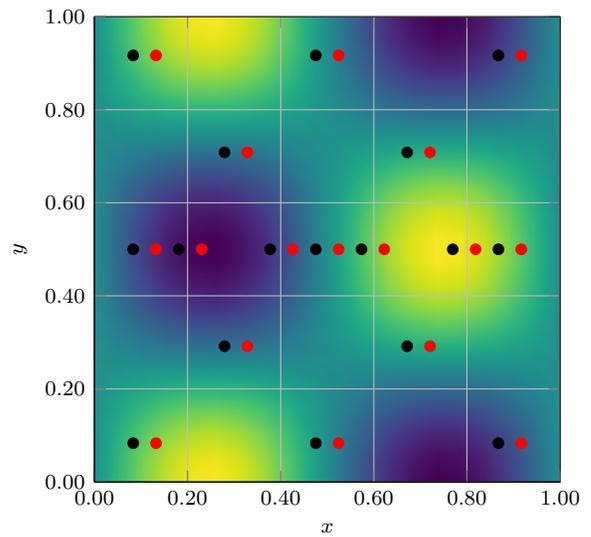
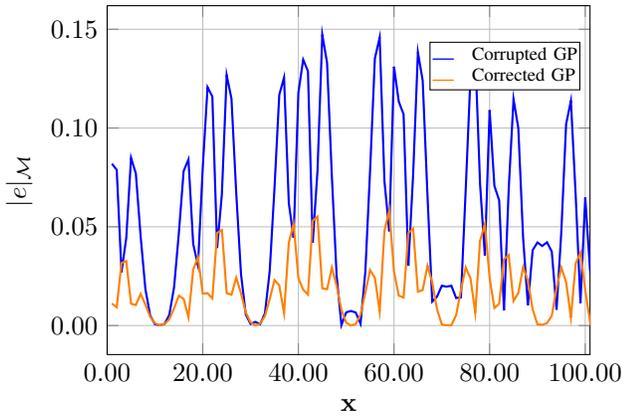
\begin{figure}
    \centering
    \begin{tikzpicture}
        \begin{axis}[
            width=0.45\textwidth,
            height=0.35\textwidth,
            xlabel={$\mathbf{x}$},
            ylabel={$|e|_{\mathcal{M}}$},
            legend pos=north east, 
            legend style={font=\small},
            grid=both,
            xmin = 0,
            xmax=101,
        ]
            \addplot[
                color=blue,
                thick,
            ]
            table [x=n_test_vector, y=error_case2, col sep=comma] {data/gp_simulation_results2.csv};

            \addplot[
                color=orange,
                thick,
            ]
            table [x=n_test_vector, y=error_case3, col sep=comma] {data/gp_simulation_results2.csv};
        \end{axis}

        \node[anchor=north east, draw, fill=white, inner sep=1pt, font=\scriptsize] 
        at (rel axis cs: 0.98, 0.98) {
            \begin{tabular}{@{}l@{}}
                \textcolor{blue}{\rule{8pt}{0.5pt}} \hspace{2pt} Corrupted GP \\
                \textcolor{orange}{\rule{8pt}{0.5pt}} \hspace{2pt} Corrected GP 
            \end{tabular}
        }; 
    \end{tikzpicture}
    \caption{Absolute error of the mean for the 2D simulation, plotted for each of the 100 test query points. The plot compares the error of the corrupted GP with the reduced error of the corrected GP.}
    \label{fig:error_trajectories_2D}
\end{figure}

\section{Discussion}
The developed correction framework introduces a principled method for addressing deterministic input errors in GP models without incurring the computational cost of full retraining. By decoupling gradient computation from the online correction loop, the method enables real-time updates to GP predictions using precomputed Jacobian and Hessian tensors. This yields improvements in both the computational efficiency (Table~\ref{tab:time_improvement_2d}) and accuracy (Table~\ref{tab:percentage_improvement_1d}) compared to conventional GP implementations that ignore location uncertainty. While this work considers the SE kernel due to its infinite differentiability property, the framework is compatible with any choice of kernel that is at least twice differentiable to support the required Hessian computations. Despite its advantages, the method faces limitations in memory scalability. The size of the stored gradient tensors grows with the number of training points and input dimensions, potentially constraining applicability to large-scale problems. This challenge motivates future work on tensor compression and structure-exploiting approximations.

\section{Conclusion}
A gradient-based correction framework has been developed to enable efficient, real-time updates to GP predictions under deterministic input perturbations. The key innovation lies in the offline computation and storage of derivative tensors, which allows the mean and covariance to be corrected without recomputing the kernel inverse.

The method demonstrates performance improvements in both computational speed and accuracy, making it suitable for dynamic environments where measurement errors evolve. Its applicability spans robotic mapping, sensor fusion, and control systems operating under uncertain localization or GPS-denied conditions. To extend the framework to high-dimensional and large-scale domains, future research should focus on reducing the memory usage of the gradient tensors through low-rank approximations or sparsity-aware representations.

\bibliographystyle{IEEETrans.bst}
\bibliography{scc,sccmaster,ECC_bib}
\end{document}